%% file: isit_fullrevise.tex
\documentclass[conference,letterpaper]{IEEEtran}
\IEEEoverridecommandlockouts
\usepackage[utf8]{inputenc} 
\usepackage[T1]{fontenc}
\usepackage{url}
\usepackage{ifthen}
\usepackage{cite}
\usepackage[cmex10]{amsmath} % Use the [cmex10] option to ensure complicance
\usepackage{tikz,tikz-cd}
\usepackage{cite}
\usepackage{amsmath,amssymb,amsfonts}
\usepackage{algpseudocode}
\usepackage{algorithm}
\usepackage{algorithmicx}
\usepackage{graphicx}
\usepackage{textcomp}
\usepackage{xcolor}
\usepackage{bbm}
\usepackage{mathrsfs}
\usepackage{amsthm}
\usepackage{subfigure}

\newtheorem{definition}{Definition}

\newtheorem{theorem}{Theorem}
\newtheorem{lemma}{Lemma}

\begin{document}
\title{Federated Recommendation System via Differential Privacy
\thanks{The work of T. Li and L. Song was supported in part by sthe Hong Kong RGC grant ECS 9048149 (CityU 21212419), the Guangdong Basic and Applied Basic Research Foundation under Key Project 2019B1515120032. The work of C. Fragouli was supported in part by the NSF grant 1740047 and the UC-NL grant LFR-18-548554.} }

\author{\IEEEauthorblockN{Tan Li$^{1,2}$, Linqi Song$^{1,2}$, Christina Fragouli$^{3}$ }
	\IEEEauthorblockA{ $^1$Department of Computer Science, City University of Hong Kong,\\
		$^2$City University of Hong Kong Shenzhen Research Institute,\\
	$^3$  University of California, Los Angeles }	 		
	\IEEEauthorblockA{Email: \{tanli6-c@my., linqi.song@\}cityu.edu.hk, christina.fragouli@ucla.edu.}
}

\maketitle

\begin{abstract}
In this paper we are interested in what we term the \textit{federated private bandits} framework, that combines differential privacy with multi-agent bandit learning. We 
explore how differential privacy based Upper Confidence Bound (UCB) methods  can be applied to multi-agent environments, and in particular to federated learning environments both in `master-worker' and `fully decentralized' settings. We provide theoretical analysis on the privacy and regret performance of the proposed methods and explore the tradeoffs between these two.
\end{abstract}
%\textit{A full version of this paper is accessible at:}
%\url{https://arxiv.org/abs/2005.06670} 


\section{Introduction}
The promise of distributed computing is to improve the efficiency and robustness of machine learning tasks by leveraging communication networks to share the computational load, leading to a compelling vision of world-wide computing \cite{1}. However, no matter how compelling this vision is,
% Add the following citations https://www.researchgate.net/publication/2718991_The_Promise_of_Distributed_Computing_and_the_Challenges_of_Legacy_Information_Systems
it cannot get realized before we address a number of challenges,
of which an important one is privacy.

In this paper, we consider privacy vs. learning trade-offs for wireless recommendation systems, that are one of the most popular learning algorithms in the consumer domain, and are considered a key application of edge-based wireless distributed systems  \cite{song2017making} \cite{song2018itw} \cite{8849556}. As a use case, we consider a multi-chain of stores, such as a fastfood chain, that make local recommendations to their customers, but then wish to aggregate the overall client responses to provide new recommendations or launch new products.  We assume that the client responses - what items they like and how much - are the private data we want to protect.

We pose our problem within the federated learning framework,  proposed by Google\cite{konevcny2016federated}, that addresses the privacy challenge by maintaining the user data locally, while 
combining learning models among the distributed agents. In particular, we consider a federated multi-armed bandit (MAB) setup, where each distributed agent could be a local store that makes recommendations, while the aggregator is the parent company. The question we explore is, can we leverage the aggregator to better inform what recommendations to make at the distributed agents, without compromising the user data privacy.

We consider in particular a distributed version of the UCB algorithm: we assume that each agent (store) makes a number of recommendations locally and calculates  a sequence of local average reward values. To combine the local models, we need to reveal the average values sequence to the aggregator, without compromising the privacy of the data. We do so by leveraging  differential privacy (DP) \cite{dwork2014algorithmic} techniques that preserve  privacy of  reward sequences. Maintaining privacy amounts to adding a form of noise, 
 which can affect which items the aggregator decides to recommend next, and which in turn can lead to a higher regret. This paper investigates  this privacy/regret trade-off.

\subsection{Related Work}
The MAB algorithm is widely used in recommendation systems due to its simplicity and efficiency \cite{li2010contextual}\cite{zeng2016online}. Auer \textit{et al.} \cite{auer2002finite} developed the UCB algorithm, which is an index-based policy relying on average reward plus an upper confidence bound. Another mainstream approach is the sampling-based approach \cite{agrawal2012analysis} that instead of computing a deterministic index, it uses a sample generated by a Bayesian estimator. 

There has been a growing literature that extends the MAB problem into multi-user settings. Liu and Zhao \cite{5535151} consider a distributed bandit problem with \textit{collisions}: choosing the same arm simultaneously leads to a reduced reward for two or more agents. Similar approaches can be found in \cite{6763073} \cite{rosenski2016multi} that utilize different matching algorithms to avoid collisions. Later work \cite{martinez2019decentralized} makes use of gossip algorithm or running consensus methods to keep an approximation of the average value between agents and their neighbors. However, few works have considered accommodating privacy considerations in the learning process.

There is also a very rich literature on differential privacy, mostly applied in deep learning\cite{abadi2016deep} and information theory fields. For decision-making problems, Tossou and Dimitrakakis present algorithms for differentially private stochastic MAB \cite{10.5555/3016100.3016190}. The work in \cite{mishra2015nearly} also investigates this problem. However, all these works operate under a single user setting. As far as we know, our federated private bandit algorithm is the first work that considers both differential privacy and communication in cooperative bandit problems.

\subsection{Main Contributions}
Our work proposes a new bandit learning framework, the \textit{federated private bandits} that combines differential privacy with multi-agent bandit learning. Our key contributions are as follows.

i) We introduce a federated private bandit framework. For each agent, we apply an ($\epsilon,\delta$) differentially private variants of the UCB scheme. Specifically, the \textit{hybrid mechanism} \cite{chan2011private} is used to track a non-private reward sequence for each agent and to output a private sum reward. The agents then use this private sum reward plus a relaxation of the upper confidence bound to update the arm index.
 
 ii) We consider two multi-agent settings: (a) the DP-Master-worker UCB (a master-worker structure): an external central node can observe all individual agent models and can return back an aggregated one to all agents; (b) the DP-Decentralized UCB (fully decentralized with networked structure): the agents average their model with their neighbors' information using a \textit{consensus algorithm} without the help of a central node. In both methods, the  real rewards are kept private from all agents.
 
 iii) We analyze  both the privacy and regret performance of our federated private UCB algorithms and characterize the influence of communication and privacy on decision making. In particular, we evaluate the trade-off between the privacy and regret.
 %We show that introducing the privacy will force us to make more explorations until have a well estimation of the unknown arm mean, which increases the total expected regret. And the connection of agents in decentralized setting also play a role in the regret performance of the whole network.  

\section{System Model and Problem Formulation}

We consider a federated recommendation system with $M$ subsystems or agents, where each agent can make recommendations to its local users. We allow the agents to communicate either through a central node (master-worker structure) or directly with their neighbors (networked fully decentralized structure), to aggregate their knowledge of the user preferences. We discuss both the `master-worker' distributed structure and the fully decentralized structure in this paper. All $M$ nodes are associated with $K$ arms (e.g., movies, ads, news, or items) from an arm set $\mathbf{A}: =\left\lbrace  1,2,...,K\right\rbrace $ that can be recommended to the users. %Moreover, we also we assume a finite number $N$ of contexts from a context set $\mathcal{C}_i =\left\lbrace  c_{1,i},c_{2,i},...,c_{n,i}\right\rbrace $.

\subsection{Federated Private Bandit Framework}
The above system model can be formulated as a $K$-armed bandit problem with $M$ distributed agents. At time slot $t$, each agent chooses and pulls an arm from the set of $K$ arms, and then the arm $j \in \mathbf{A}$ chosen by agent $i \in [M]$ generates an i.i.d. reward $r_{i,j}(t)$ from a fixed but unknown distribution at time $t$. We denote by $\mu_{i,j}$ the unknown mean of reward distribution. In our model, the reward distribution of each arm is the same for each agent, i.e., for all arms $1\leq j \leq K$, $\mu_{1,j}=\mu_{2,j}=...=\mu_{i,j}=...=\mu_{M,j}$, and thus in the rest of the paper we use $\mu_j$ for simplicity.

The arm that agent $i$ plays at time $t$ is denoted as $a_i(t)\in \mathbf{A}$. Let $q_i(t)$ be the communication message sent by agent $i$ and $q_{-i}(t)$ be the messages received by agent $i$ at time $t$. Here, messages can be learning model parameters which will be specified later. Then the policy $\pi_i(t)$ for agent $i$ can be viewed as a mapping from the collected history set to the action set. That is, $\pi_i(t):\mathit{H}_{i}(t)\rightarrow\mathbf{A}$, where the history $\mathit{H}_{i}(t)$ gathers actions, rewards, and message exchange of the past $\mathit{H}_i(t)=\left\lbrace{(a_i(1),r_{i,a_i(1)}(1),q_{-i}(1)),...,(a_i(t-1),r_{i,a_i(t-1)}(t-1),}\right.\\\left.{ q_{-i}(t-1)) }\right\rbrace  $. The overall objective of the $M$ agents is to maximize the expected sum reward over a finite time horizon $T$:
%\begin{eqnarray}
$\mathbb{E}[\sum_{t=1}^{T}\sum_{i=1}^{M}r_{i,a_i(t)}(t)]$.
%\end{eqnarray}
Without loss of generality, we can assume that $\mu_1$ is always the best arm for each agent. Then the suboptimality gap can be defined as $\Delta_{j} : = \mu_{1} - \mu_{j}$ for any arm $j\not=1$. %Recall that all $M$ agents act over the same $K$ arms, we have $\Delta_{1,j}=\Delta_{1,j}=...=\Delta_{i,j}=...=\Delta_{M,j}$, we may use $\Delta_j$ for simplicity. 
Let $n_{i,j}(t)$ be the number of times arm $j$ is pulled by agent $i$ up to time $t$, then the number of times arm $j$ is pulled by all the agents in the network up to time $t$ can be calculated as $n_{j}(t):=\sum_{i=1}^{M }n_{i,j}(t)$. 

%Let $P(N)$ denote the set of possible permutations of the $N$ arms, the optimal policy is defined as 
%\begin{equation}
%\mathbf{k}* \in \underset{\mathbf{k}\in P(N)}{argmax}\sum_{i=1}^{M}\mu_{i,k_i}
%\end{equation}
The learning goal is to minimize the overall expected regret, which is defined as the expected reward difference between the best arm and the online learning policies of the agents. For policies with action $a_i(t)$ ($\forall i\in [M], \forall t$), the overall expected regret is defined as
\begin{equation}
\mathit{R}(T) = TM\mu_1-\mathbb{E}[\sum_{t=1}^{T}\sum_{i=1}^{M}\mu_{i,a_i(t)}(t)]
=\sum_{j=2}^{K}\Delta_j\mathbb{E}[n_j(T)]
\end{equation} 
%\end{definition}

%{equation}
%\mathit{R}(T) = %T\sum_{i}\mu_{i,k_i**}-\mathbb{E}[\sum_{t=1}^{T}\sum_{i=1}^{M}r_{i,a_i(t)}(t)]
%\end{equation} 

%Here, we do allow agents to communicate with the central node and there of course comes at a cost and contributes to regret. Communication costs are justified because distributed algorithms require a certain amount of information exchange over multiple time slots. This cost will depend on the specific algorithm. In our paper, we discuss the information exchange under both centralized and decentralized situation. We assume information exchange though the central server or among their neighbors will generate $C$ cost which will be specified latter, and let $m(t)$ be the number of times it happens in time $t$. Then, the expected regret can be rewrite as:
%\begin{equation}
%\mathit{R}(T) = TM\mu_1-\mathbb{E}[\sum_{t=1}^{T}\sum_{i=1}^{M}\mu_{i,a_i(t)}(t)]
%+C\mathbb{E}[m(T)]
%\end{equation} 
\subsection{Differential Privacy}\label{DP}
We use differential privacy as our privacy metric and briefly review some background material in the following.

\begin{definition}[Differential Private Bandit Algorithm]
%A randomized algorithm $\pi$ is ($\epsilon,\delta$)-differentially private if for any two neighboring data sequences $D$ and $D'$ and for all set $\mathcal{S}\subseteq \mathcal{A}$:
%\begin{equation}
%\Pr\{\pi(D)\in\mathcal{S}\}\leq \exp{(\epsilon)}\Pr\{\pi(D')\in\mathcal{S}\}+\delta
%\end{equation}
\iffalse
A bandit algorithm $\pi$ is ($\epsilon,\delta$)-differentially private if for all two neighboring history sequences $\mathbf{H}(t)=\left\lbrace H_1(t),...,H_M(t)\right\rbrace $ and  $\mathbf{H}'(t)=\left\lbrace H'_1(t),...,H_M'(t)\right\rbrace $ and for all set $\mathcal{S}\subseteq \mathcal{A}$, the
following holds:
\begin{equation}
\Pr(\mathbf{a}(t)\in\mathcal{S}|\mathbf{H}(t))\leq \exp{(\epsilon)}\Pr(\mathbf{a}(t)\in\mathcal{S}|\mathbf{H}'(t))+\delta
\end{equation}
\fi

A bandit algorithm $\pi_i$ for agent $i$ is ($\epsilon,\delta$)-differentially private if for all two neighboring reward sequences $\mathbf{r}(t)=\left\lbrace r_{i,a_i(1)}(1),...,r_{i,a_i(t)}(t)\right\rbrace $ and  $\mathbf{r}'(t)=\left\lbrace r'_{i,a_i(1)}(1),...,r'_{i,a_i(t)}(t)\right\rbrace $ (i.e., that differ on at most 1 position), for all subsets $\mathcal{S}\subseteq \mathcal{A}$, and for all measurable image subsets $\mathcal{Q}$ of $q_i(t)$, the following holds:
\begin{equation}
\begin{array}{llll}
\Pr\{a_i(t)\in\mathcal{S},q_i(t)\in\mathcal{Q}|\mathbf{r}(t)\}\leq  \\ \exp{(\epsilon)}\Pr\{a_i(t)\in\mathcal{S},q_i(t)\in\mathcal{Q}|\mathbf{r}'(t)\}+\delta.
\end{array}
\end{equation}
We say the algorithm of the system is ($\epsilon,\delta$)-differentially private if (2) holds for all agents.
\end{definition}  	
%Here, $\mathbf{a}(t)=\{a_1(t),...,a_M(t)\}$, $\mathcal{A}$ is the output action set of the algorithm and the neighboring sequences mean that $\mathbf{H}(t)$ and $\mathbf{H}'(t)$ differ at most one time step.
Intuitively, for our bandit problem, if the reward $r_{i,j}(\tau)$ for arm $j$ and agent $i$ is the private information, the definition above implies that we want the algorithm to protect the arm's reward realization $r_{i,j}(\tau)$ against an adversary even if the adversary can observe the output actions $a_i(1),a_2(2),\ldots,a_i(t)$, the transmitted information $q_i(1),q_i(2),\ldots,q_i(t)$, and other reward realizations. 

A commonly used differential privacy scheme is the \textit{Laplace} mechanism, which simply adds a \textit{Laplace} noise $N\sim~Lap(\frac{s}{\epsilon})$ to the private data communicated. In our problem, we employ a more sophisticated  differential privacy mechanism, termed the hybrid mechanism, that we briefly describe next.

%\paragraph{Laplace Mechanism}
%A natural way to get privacy is to add a Laplace noise $N\sim~Lap(\frac{s}{\epsilon}) $ to the output of the algorithm. The density function of $Lap(\lambda)$ with parameter $\lambda$ is $\frac{1}{2\lambda}\exp(\frac{-|x|}{\lambda})$. And $s$ is the $l_1-$ sensitivity of the algorithm\cite{dwork2014algorithmic}, $\epsilon$ is a parameter for privacy.
%Comment:say why better than simple laplace

The Hybrid Mechanism\cite{chan2011private} is a \textit{tree based aggregation} scheme that releases private statistics over a data sequence. Consider a reward sequence $\mathbf{r}=(r(1),r(2),...,r(T))$, where at each time $t$ a new $r(t)\in[0,1]$ is inserted. Assume we want to output the partial (up to time $t$) sum $s(t)=\sum_{i=1}^{t}r(i)$ while ensuring that the sequence $\mathbf{r}$ is $(\epsilon,\delta)$-private.
The Hybrid mechanism outputs partial sums at times $t=2^k,k=1,2,..$. 
For the time period $2^k$ and $2^{k+1}$, 
the mechanism
constructs a binary tree $B(t)$ that has as leaves the inputs $r(i)$,  all other nodes store partial sums, and the root node contains the sum from $2^k$ to $2^{k+1}-1$. The mechanism 
outputs a private sum $L(t)$ by adding a \textit{Laplace} noise of scale $\frac{1}{\epsilon}$, i.e., $Lap(\frac{1}{\epsilon})$ to a set of nodes that ``cover'' all the inputs.
As compared to the straightforward approach of adding noise to each sample $r(i)$, this method enables to output partial sums that satisfy the same differential privacy gurantees adding overall less noise - 
indeed there is only a logarithmic amount of noise added for any given sum because of the logarithmic tree depth. %\cite{chan2011private} also assures an additive error of $O(\frac{\log^{1.5}T}{\epsilon})$ per reward sequence.

%Then the \textit{Binary Noisy Sum} mechanism is used to build a binary tree $B$, whose leaf nodes store the input $r(i)$ while all other nodes store the sum $s(t)$, with the root maintaining the sum from $2^k$ to $2^{k+1}-1$. There will be only a $\frac{\log t}{\epsilon}$ scale noise added for any partial sum by performing each step \textit{Laplace} mechanism.

\section{Federated Private Multi-Armed Bandits}

In this section, we present two algorithms for the federated private bandit problems under different settings and provide their performance analysis. Our algorithms combine the non-private UCB algorithm\cite{auer2002finite} with the Hybrid $(\epsilon,\delta)$ differential privacy technique.

In the UCB algorithm, at time slot $t$, each arm $j$ of agent $i$ updates an estimate of the index $I_{i,j}(t)$, which is calculated as the sum of the empirical mean $Y_{i,j}(t)$ and an upper confidence bound:
%\begin{equation}
$I_{i,j}(t)=Y_{i,j}(t)+\sqrt{\frac{2\text{log}t}{n_{i,j}(t)}}$.
%\end{equation}
Here, $Y_{i,j}(t)={y_{i,j}(t)}/{n_{i,j}(t)}$, $y_{i,j}(t)$ is the sum of observed rewards and $n_{i,j}(t)$ is the total number of times that arm $j$ has been pulled until time $t$.

To achieve differential privacy, we apply a DP mechanism as shown in Figure. 1. In particular, we instantiate the hybrid mechanism $H_{i,j}$ for each arm $j$ at each agent $i$, which keeps track of the non-private empirical mean $Y_{i,j}$ and outputs a private mean $X_{i,j}$. Here $X_{i,j}=s_{i,j}/n_{i,j}$ and $s_{i,j}$ is the private sum reward. The agents select actions based on the private mean $X_{i,j}$ instead of the empirical mean $Y_{i,j}$, thus ensuring that the actions are also differentially private. 
\begin{figure}[ht]	
	\centering
	\includegraphics[scale=0.20]{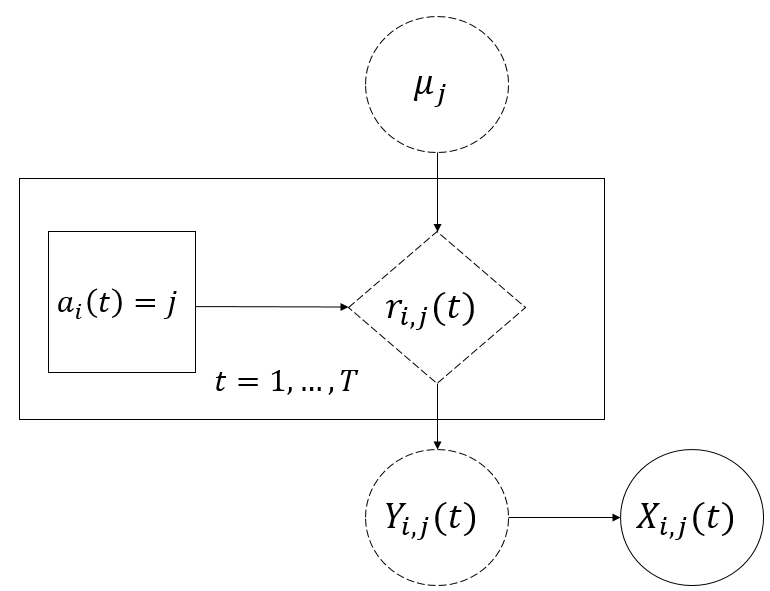}
	\caption{Graphical model for the hybrid mechanism $H_{i,j}$ of agent $i$. }
	\label{fig:label}
\end{figure}

We present two federated learning algorithms. The first, termed DP-Master-worker UCB algorithm, employs the DP mechanism to compute the individual arm index that consists of a private mean as well as an additional privacy-induced uncertainty. Then, the central node aggregates and returns back an aggregated index that will be used arm selection. The second, termed DP-Decentralized UCB algorithm, employs the same DP mechanism, but the agents estimate the index by averaging their neighbors' input.
%In both cases the randomness introduced to the observed rewards will leads to a noisy estimation of the average  and will affect the decision accuracy.

%grossly corrupt the estimation of the means of the arms, as well as the decision accuracy. %We address this issue by increasing the number of exploration phase, that means, the rounds to estimate these means. The exact details are different for the two algorithms that are discussed in the respective sections. 

\subsection{DP-Master-worker UCB algorithm}
In Algorithm 1, each arm $j$ of each agent $i$ uses the DP mechanisms $H_{i,j}$ (shown in Figure 1) to maintain a private total reward $s_{i,j}$. The communication phase begins when the counter $\eta =2^p~ \text{for}~p=1,2,...$. The individual arm index of each agent $i$ is first updated using the private mean, the upper confidence bound and the additional noise due to privacy (Line 12). Then, the central node averages all the private indices to compute an average index which leads to the same best arm selection for all $M$ agents. Each agent $i$ starts from the common index and privately updates it. For each agent, if an arm is pulled for the $p^{th}$ time consecutively (without switching to any other arms in between), it will also be played for the next $2^p$ time slots. %If the agent selects different best arms, the counter $\eta $ will be reset to 1 (Line 12-14). This ensures that each arm will be fully pulled to get an accurate estimate of the mean. %For any arm $j\neq a_i(t)$, we add a dummy reward 0 to their DP mechanisms.
%如果这个arm被play过的次数很少，则对每个arm的reward realization估计不准确，连续两个p可能会选择不同的arm，这样就和普通的UCB算法基本相似，接近于每个timeslot都根据index选择arm并play. 随着time t 增加，arms 都得到了充分估计，则有可能连续两次出现play同一个arm（这时我们也认为该arm很可能是最优的arm），这种情况就加强利用，play当前的best arm 2^p 次。 

 %If the number of times the arm has been played is very small, the estimate of reward realization of is not accurate. In two consecutive p, the algorithm may choose different arm, which is basically similar to the original UCB algorithm, that mean, ateach time slot we select and play one arm based on index. As time t increases, the arms are fully estimated, it is possible to play the same arm twice in succession (at this time we also think that the arm is likely to be the optimal arm), so the promising best arm should be exloiate more , we play 2 ^ p times.

\begin{algorithm}
	\caption{DP-Master-worker UCB algorithm}
	\begin{algorithmic}[1]
		%\Require Input
		%\Ensure Output
		
		\State \textbf{Initialization}: Set $t~=~0$ and counter $\eta~=~1$; 
		\State For each arm $j,~1\leq j\leq K$ of each agent $i,~1\leq i\leq M$, instantiate DP mechanisms $H_{i,j}$.
		\State \textbf{Input:} The differential privacy parameter $\epsilon $;
		%\Function {"FuncName"}{}
		
		\While {$t \leq T$} 
	    \For {agent $i$ to $M$}	
		\If {$t \leq K$} 
		\State Play arm $a_{i}(t) = t$, observe reward $r_{i,a_{i}(t)}(t)$
		\State Insert $r_{i,a_{i}(t)}(t)$ to the DP mechanism $H_{i,a_{i}(t)}$
		\EndIf
		%\For{each epoch $l$}  
		%\State \textbf{Exploration Phase}: Each agent $i,~1\leq i\leq M$, plays each arm $j,~1\leq j\leq K$ $c$ number of times;
		\If {$\eta = 2^p$ for $ p = 0,1,...$}
		\State Update total reward $s_{i,j}(t)$ using $H_{i,j}$ 
		\State Update $\upsilon _{i,j} = \frac{ 1}{\epsilon}log\frac{1}{\delta}\log^{1.5}n_{i,j}(t)$
		\State Update index $I_{i,j}(t)=X_{i,j}+\sqrt{\frac{2\log t}{n_{i,j}(t)}}+\frac{\upsilon_{i,j}(t)}{n_{i,j}(t)}$
		\State \textit{/*Begin communication phase}
		\State Send index $I_{i,j}(t)$ to the central node
		\State Receive the averaged index $I^{avg}_{j}(t)$ of $j$ arms
		\State \textit{/*End communication phase}
		\State Pull best arm $a^*_{i}(t)={argmax}_jI^{avg}_j(t)$ 
		%\Endfor 
		%\For {the central node}
		% Receive the individual index $I$
		\If {$a^*_i(t) \neq a^*_i(t-1) $}
		\State Reset $\eta = 1$;
		\EndIf
		\Else\State {$a^*_i(t) = a^*_i(t-1) $}
		%\If  {communication round occurs}
		%\State Update the arm mean using the average value of all $M$ agents $\overline{X}_{i,l}=\frac{1}{M}\sum_{i=1}^{M}(\frac{s_{i,k}(l)}{n_{i,k}(l)}+\frac{v_{i,k}(l)}{n_{i,k}(l)})$  
		%\Else   
		%\State Update the arm mean using the local information
		%$\overline{X}_{i,l}=\frac{s_{i,k}(l)}{n_{i,k}(l)}+\frac{v_{i,k}(l)}{n_{i,k}(l)}$
		%\EndIf
		\EndIf	
		\State Play arm $a^*_i(t)$, observe the reward $r_{i,a^*_i(t)}(t)$
		\State Insert $r_{i,a^*_i(t)}(t)$ to the DP mechanism $H_{i,a^*(t)}$
		%\State \textbf{Exploitation Phase}:play arm $a_{i}(l)$ for $2^l$ times;
		
		\State Update $t = t+1 , \eta=\eta+1$

		%\EndFor
		\EndFor
		\For {The central node}
		\State \textit{/*Begin communication phase}
		\State Receive index sequence $\left\lbrace I_{1,j}...I_{M,j}\right\rbrace $ of $j$ arms
		\State Compute and return back $I^{avg}_j=\frac{1}{M}\sum_{i=1}^{M}I_{i,j}$
		\State \textit{/*End communication phase}
		\EndFor
		\EndWhile
		
	\end{algorithmic}
\end{algorithm}

We next analyze the algorithm performance. Theorem 1 provides the privacy performance of Algorithm 1.

\begin{lemma}[Privacy error bound]
\label{lemma:errorbound}
The error between the empirical mean $Y_{i,j}$ and private mean $X_{i,j}$ after $n_{i,j}$ times of plays is bounded as $|Y_{i,j}-X_{i,j}| \leq h_{n_{i,j}}$ with probability at least $1-\delta$, where $h_{n_{i,j}}$ is the error incurred by the private mechanism calculated as $h_{n_{i,j}}=\frac{1}{\epsilon}\cdot \log ^{1.5}({n_{i,j}})\cdot \log \frac{1}{\delta}\cdot \frac{1}{n_{i,j}}$.
\end{lemma}
\begin{proof}
This follows directly from the Fact 1 (Appendix.A\cite{2005.06670}) that the hybrid mechanism remains $(\epsilon,\delta)$-DP after any number $n_{i,j}$ of plays since each time only one arm is pulled, that will affect only one mechanism. 
\end{proof}
%Moreover, the only quantity using the output private reward $s_{i,j}$ is the private mean $X_{i,j}$ thus it can also achieve differential privacy.
\begin{theorem}[Privacy of Algorithm 1]
	Algorithm 1 is $\epsilon$- differential private after $T$ timeslots with $\delta=T^{-4}$. %after number of $n$ of plays.
\end{theorem}
\begin{proof}	
 Proposition 2.1 of \cite{dwork2014algorithmic} proves the \textit{post-processing} property of DP mechanisms: the composition of a mapping $f$ with an $(\epsilon, \delta)$- differentially private algorithm is also $(\epsilon, \delta)$ differentially private. Using Lemma~1, the hybrid mechanism is $(\epsilon, \delta)$- differentially private. Moreover, our Algorithm 1 can be seen as a mapping from the averaged output of the hybrid mechanism to the action. This completes our proof.
\end{proof}

 %Furthermore, Theorem 3.16 in \cite{dwork2014algorithmic} indicates the \textit{group privacy} for $(\epsilon, \delta)$ private mechanism: the composition of $k$ differentially private mechanisms, where the $i$-th mechanism is  $(\epsilon_i, \delta_i)$-differentially private, for $1\leq i\leq k$, is $(\sum_i\epsilon_i,\sum_i\delta_i)$. Using above results completes the proof.

Theorem 2 gives the regret of  Algorithm 1. Here we only give a proof sketch,  the complete proof can be found in Appendix.C\cite{2005.06670}. 
\begin{theorem}[Regret of Algorithm 1]
\label{thm:1}
	The learning regret of Algorithm 1 is
\begin{equation}
	\begin{array}{ll}
R^C(T)\leq MK\Delta_{max}
(4+\max[(\frac{8\log{T}}{\epsilon(1-\beta_0) \Delta_{min}} )^{2.25}, \left \lceil \frac{8\log{T}}{\Delta_{min}^2 \beta_0^2} \right \rceil]\notag
\end{array}
\end{equation}
for some $0< \beta_0< 1$, where $\Delta_{max}=\max\left\lbrace \Delta_j\right\rbrace $, $\Delta_{min}=\min\left\lbrace \Delta_j\right\rbrace $, $\epsilon$ is the parameter for $(\epsilon,\delta)$ privacy, $\delta=T^{-4}$.% and $C$ is a constant that captures the communication cost.
\end{theorem}
\textit{Proof outline.} The regret incurred during the time horizon $T$ is caused by playing suboptimal arms. We first bound the amount of error between the private and empirical means that are caused by the DP mechanism. Using this bound and Lemma~\ref{lemma:errorbound}, we estimate the number of times that we play suboptimal arms. We show that after a sufficient number of times $O(\frac{MK\log^{1.5}T}{\epsilon\Delta^2_{min}})$, a suboptimal arm will not be selected with high probability. 

\textit{Remark:} Through the central mode we obtain $O(MK\log^{2.25}(T))$ regret. 
%The federated learning setup incurred $O(MK\log T)$ regret since the agents need to communicate and recompute their index with the central node. 
The DP mechanism mainly increases the exploration rounds. If we do not use the DP mechanism, the $O(\log^{2.25}T)$ term vanishes. We note that after $\frac{8\log T}{\Delta_{min}^2}$  plays, the suboptimal arms will be selected with low probability, and we can achieve
%\begin{equation}
%R^C(T)\leq MK(\Delta_{max}+C(1+\log T))(\frac{8\log T}{\Delta_{min}^2}+4)
%\end{equation}
a $O(MK\log T)$ regret. Note that in Theorem \ref{thm:1}, the parameter $\epsilon$ reflects the trade off between privacy and regret, where the privacy increases as $\epsilon$ decrease.

\subsection{DP-Decentralized UCB algorithm}

\begin{algorithm}
	\caption{DP-Decentralized UCB algorithm}
	\begin{algorithmic}[1]  
		%\Require Input
		%\Ensure Output
		
		\State \textbf{Initialization}: Set $t=0$;  $\hat{\mathbf{n}}_j(0)=\left\lbrace \hat{n}_{1,j}(0),...,\hat{n}_{M,j}(0)\right\rbrace $, $\hat{\mathbf{s}}_j(0)=\left\lbrace\hat{s}_{1,j}(0),...,\hat{s}_{M,j}(0)\right\rbrace$
		\State For each arm $j,~1\leq j\leq K$ of each agent $i,~1\leq i\leq M$, instantiate DP Mechanisms $H_{i,j}$.
		\State \textbf{Input:} The differential privacy parameter $\epsilon $; matrix $P$ represents the network structure; $\rho\geq1$;

		%\Function {"FuncName"}{}
		%if
		\While {$t \leq T$} 
		\For {agent $i$ to $M$}
		\If {$t \leq K$} 
		\State play arm $a_{i}(t) = t$, observe the reward $r_{a_{i}(t)}(t)$
		\State Insert $r_{a_{i}(t)}(t)$ to the DP mechanism $H_{i,a_{i}(t)}$
		\Else 
		%\For {Each agent $i$ and each arm $j$}
		%\State Update the total sum $s_{i,k}(t)$ using hybrid mechanism $H_{i,k}$ 
		
		\State \textit{/*Begin the communication phase}
		\State Update the estimated play numbers:
		\State $\hat{\mathbf{n}}_j(t)=P\hat{\mathbf{n}}_j(t-1)+P\mathbf{\eta}_j(t-1)$ 
		%$\leftarrow$ the $i$-th entry of $\mathbf{\eta}_j(t-1)$ is 1 if at time $t-1$ node $i$ play arm $j$
		\State Update the additional private error term:
		\State $\hat{\upsilon }_{j,j}(t) = \frac{ 1}{\epsilon}\log\frac{1}{\delta}\log^{1.5}\hat{n}_{i,j}(t)$
		\State Update the estimated total rewards: 
		\State $\hat{\mathbf{s}}_j(t)=P{\hat{\mathbf{s}}}_j(t-1)$ %$\leftarrow$ the $i$-th entry of $\hat{\mathbf{s}}_j(t-1)$ is the private reward $s_{i,j}(t-1)$
		\State \textit{/*End the communication phase.}
		\State Update the arm index :
		\State $I_{i,j}(t)=\hat{{X}}_{i,j}
		+\sqrt{2\rho \frac{\hat{n}_{i,j}(t)+c_i}{M\hat{n}_{i,j}(t)}\cdot \frac{\log t}{\hat{n}_{i,j}(t)}}
		+\frac{\hat{\upsilon} _{i,j}(t)}{\hat{n}_{i,k}(t)}$
		\State Select the best arm %$a_{i}(t)={argmax}_a\overline{X}_{i,k}(t)$
		$a_{i}(t)={argmax}_jI_{i,j}(t)$
		\State Observe the reward $r_{a_{i}(t)}(t)$
		\State Insert $r_{a_{i}(t)}(t)$ to the DP mechanism $H_{i,a_{i}(t)}$
		\State Update $s_{i,a_i(t)}(t)$ using DP mechanism $H_{i,a_i(t)}$

		\State $t=t+1$
			
		\EndIf
		\EndFor	
		\EndWhile
		
	\end{algorithmic}
\end{algorithm}

In Algorithm 2, the agents average their model with their neighbors’ models at each time $t$, instead of aggregating their values with the help of a central node. We assume that each agent maintains a  bi-directional communication with a set of neighboring agents. We consider Gaussian distributions for each arm's reward, i.e., the reward at arm $j$ is sampled from a Gaussian distribution with mean $\mu_{i,j}$ and variance $\sigma^2$. We assume that the variance $\sigma^2$ is known and is the same at each arm. We use a \textit{consensus algorithm} that captures the effect of the additional private information an agent receives through communication with other agents. We represent the network as a graph where nodes are agents and edges connect neighboring agents. A discrete-time consensus algorithm can be expressed as:
\begin{equation}
\mathbf{x}(t+1)=P\mathbf{x}(t),
\end{equation}
where $x$ is the quantity we want the agents to agree on, and $P$ is a row stochastic matrix given by
\begin{equation}
P = \mathit{I}_M -\frac{\kappa}{d_{max}}L.
\end{equation}
Here, $\mathit{I}_M$ is the identity matrix with order M, $d_{max}=\max_i \deg(i),i\in \left\lbrace 1,...,M\right\rbrace $ and $\deg(i)$ is the degree of agent $i$. $\kappa\in[0,1]$ is a step size parameter and $L$ is the Laplacian matrix of this communication graph. Without loss of generality, we assume that the eigenvalues of $P$ are ordered as $\lambda_1=1\ge\lambda_2\geq...\geq\lambda_M\ge-1$.

For our federated private MAB problem, we use the following definitions, that are similar to the definitions in Algorithm 1. Let $\hat{s}_{i,j}$ be the estimated total private reward, $\hat{y}_{i,j}$ be the estimated total true reward of arm $j$ at agent $i$, and $\hat{n}_{i,j}$ be the estimated total number of times that the arm $j$ has been played by agent $i$.
Let $\hat{{X}}_{i,j} = {\hat{s}_{i,j}}/{\hat{n}_{i,j}}$ be the estimated private mean, and $\hat{{Y}}_{i,j} = {\hat{y}_{i,j}}/{\hat{n}_{i,j}}$ be the estimated empirical mean.

Without taking into account differential privacy, the consensus algorithm will update $\hat{y}_{i,j}$ and $\hat{n}_{i,j}$ as follows:
\begin{eqnarray}
\label{eq9}\mathbf{\hat{n}}_{j}(t+1)=P\mathbf{\hat{n}}_{j}(t)+P\mathbf{\xi}_j(t)\\
\label{eq10}\mathbf{\hat{y}}_{j}(t+1)=P\mathbf{\hat{y}}_{j}(t)+P\mathbf{r}_j(t),
\end{eqnarray}
where ${\xi}_{i,j}(t)=I(a_i(t)=j)$, indicating if arm $j$ is played by agent $i$ at time slot $t$; $r_{i,j}(t)$ is the reward with respect to the action which is generated by the distribution $N(\mu_j,\sigma^2)$. $\mathbf{\hat{n}}_{j}(t),\mathbf{\xi}_j(t), \mathbf{\hat{y}}_{j}(t), \mathbf{r}_j(t)$ are vectors that connect the values $\hat{n}_{i,j}(t), \xi_{i,j}(t), \hat{y}_{i,j}(t), r_{i,j}(t)$ for $i= 1,...,M$ respectively. We note that under our DP mechanism, an agent can not observe the reward sequences. Thus, we use the following equation to update the private total rewards instead of \eqref{eq10}:
\begin{equation}
\mathbf{\hat{s}}_{j}(t+1)=P\mathbf{\hat{s}}_{j}(t).
\end{equation}
The above equation captures the fact that only the private total reward $\mathbf{\hat{s}}_{j}(t)$ can be broadcasted through the network graph, not $\mathbf{r}_j(t)$. We still keep \eqref{eq9}  because we only aim to keep the reward values private and not the numbers $\mathbf{\hat{n}}_{j}$. 

Each arm $j$ of each agent $i$ uses the analogous DP mechanisms $H_{i,j}$ in the the Algorithm 1 to maintain a private total reward. The communication phase occurs at each timeslot to update the estimate play numbers $\hat{n}_{i,j}(t)$ and the total reward $\hat{s}_{i,j}(t)$ using (5) (7). Agent $i$ selects the arm with the maximum index denoted as:
\begin{equation}
%\begin{array}{ll}
I_{i,j}(t)=\hat{{X}}_{i,j}
+\sqrt{2\rho \frac{\hat{n}_{i,j}(t)+c_i}{M\hat{n}_{i,j}(t)}\cdot \frac{\log t}{\hat{n}_{i,j}(t)}}
+\frac{\hat{\upsilon} _{i,j}(t)}{\hat{n}_{i,k}(t)},
%\end{array}
\end{equation} where $c_0,c_i$ are parameters representing the network stricture and $\rho>1$ is the exploration parameter. 
From (8) we notice that the estimation performance, the network structure, and the exploration parameter, all affect the learning performance.

\begin{theorem}[Privacy of Algorithm 2]
	 Algorithm 2 is $(\epsilon,\delta)$- differentially private after $T$ timeslots with $\delta = \frac{1}{2}T^{-\rho}$.
\end{theorem}
The proof of Theorem 3 is similar to that of Theorem 1.
%\textit{Proof.}  Using Lemma 1 and the \textit{post-processing} of DP mechanisms, we can complete the proof.

\begin{theorem}[Regret of Algorithm 2]
	The learning regret of Algorithm 2 is
\begin{equation}\label{eq17}
\begin{array}{ll}
R^D(T) \leq \frac{2MK\rho\Delta_{max}}{\rho -1}+\sum_{i=1}^{M}\sum_{j>1}^{K}\max[(\frac{2+2\rho\log{T}}{\epsilon(1-\beta_0) }) ^{2.25}, \notag\\
\left \lceil\frac{c_0}{\beta_0^2}+\frac{8\sigma^2\rho(1+c_i)\log T}{\beta_0^2\Delta_{j}} \right \rceil]
\end{array}
\end{equation}	
for some $0< \beta_0 < 1, ~\rho\geq1$, where $\Delta_{max}=\max\left\lbrace \Delta_j\right\rbrace $, $\Delta_{min}=\min\left\lbrace \Delta_j\right\rbrace $ and $\epsilon$ is the parameter for $(\epsilon,\delta)$ privacy, $\delta = \frac{1}{2}T^{-\rho}$ and $c_i$, $c_0$ are parameters of the network graph.
\end{theorem}

\textit{Proof outline}. The regret is mainly caused by the estimated variance due to communication and the privacy requirements. By using Lemma 1 and Lemma 2 (provided in Appendix.B\cite{2005.06670}), we first bound the amount of error between the estimated private mean and empirical mean.  We note that the communication cost is also reflected in this bound. Using this bound, we estimate the number of times  suboptimal arms are selected, and complete the proof. The complete proof can be found in Appendix.D\cite{2005.06670}.

\textit{Remark:} From Theorem 4 we obtain $O(MK\log^{2.25}T)$ regret. Both the communication and the privacy mechanism result in an  expansion of the exploration phase. %The second term in \eqref{eq17} captures this effect.
%$$
%\max[(\frac{8\text{log}{T}}{\epsilon(1-\beta_0) \Delta_{j}}) ^{2.25},\left \lceil \frac{c_0}{\beta_0^2}+\frac{8\sigma^2\rho(1+c_i)\text{log}T}{M\beta_0^2\Delta_{j}^2} \right \rceil]
%$$ %If we do not perform DP mechanism, the $O(\log^{2.25}T)$ term will not exist. Under this situation, after $\frac{8\log T}{\Delta_{min}^2}$ times of plays, the suboptimal arms will be selected with low probability, and we have
%\begin{equation}
%R(T)\leq MK(\Delta_{max}+C(1+\log T))(\frac{8\log T}{\Delta_{min}^2})
%\end{equation}
The DP mechanism leads to an additional $O(\log^{2.25}T)$ regret with the parameter $\epsilon$ inversely proportional to the regret. The federated learning setup introduces constants $c_0$ and $c_i$ into the regret which depend on the network topology. In particular, $c_0$ is proportional to the network scale and $c_i$ depends on the number of neighbors of agent $i$. The sparser the network connection, the larger the $c_i$ and the regret. A larger exploration parameter $\rho$ also implies more exploration rounds. 

\section{Experiments}
In this section, we mainly perform numerical simulations to verify and analyze the performance of Algorithm 2. We choose $M=20$ and $K=10$. % with arm mean $\mu=[1,0.8,0.8,0.6,0.6,0.4,0.4,0.2,0.2,0.1]$. 
The 20 agents are connected according to a \textit{cycle graph} which is a fully decentralized setting. %Due to the space limitation, the experimental results are shown in our Appendix.
Figure 1 shows the impact of varying the privacy parameter $\epsilon$ in $\{1.5,2,5\}$ with fixed $\rho=2$. We can see that the regret increases with $\epsilon$. Figure 2 shows the impact of varing the exploration parameter $\rho$ in $\{1.2,2,4\}$ with fixed $\epsilon = 2$. Again as expected the regret increases with $\rho$. These results demonstrate the tradeoff between the regret (recommendation accuracy) and privacy.
\begin{figure}[htbp]
	\centering
	\subfigure[Regret as a function privacy parameter $\epsilon$.]{
		\includegraphics[width=3.8cm]{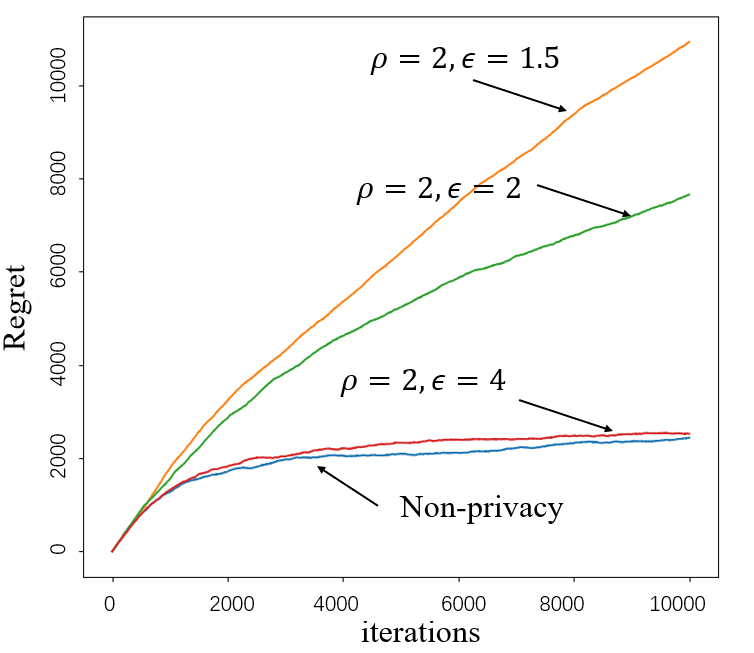}
		%\caption{fig1}
	}
	\quad
	\subfigure[Regret as a function of exploration parameter $\rho$.]{
		\includegraphics[width=4.1cm]{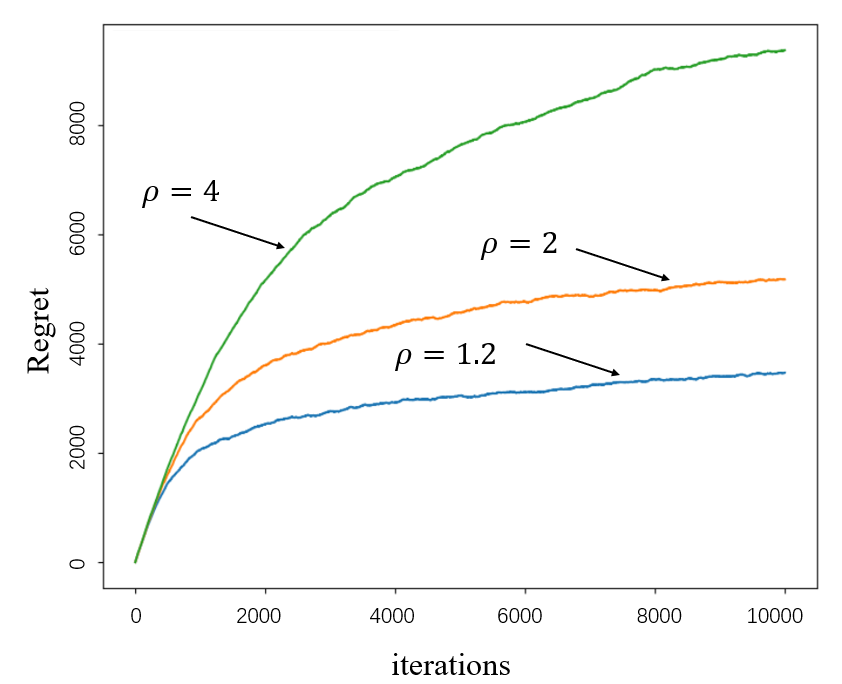}
	}
	\caption{Regret performance of Algorithm 2.}
\end{figure}
%\begin{figure}[ht]
%	
%	\centering
%	\includegraphics[scale=0.18]{isit1.png}
%	\caption{Graphical model for the Hybrid mechanism $H_{i,j}$ of agent $i$. }
%	\label{fig:label}
%\end{figure}

\section{Conclusion}

In this paper, we proposed a distributed MAB framework for recommendation systems that incorporates differential privacy. At each distributed agent, we use an ($\epsilon,\delta$) differentially private variant of UCB scheme to ensure that agents do not reveal information on the reward values. We designed algorithms for two multi-agent settings: the DP-Master-worker UCB algorithm and the DP-Decentralized UCB algorithm each capturing a different communication network connecting agents. We analyzed both the privacy and regret performance and showed how the need for communication and privacy can influence the decision making performance of the agents.

%In this paper, we proposed a distributed multi-armed bandit frame work for recommendation systems that incorporates differential privacy. At each distributed agent, we use an ($\epsilon,\delta$) differentially private variant of UCB scheme to ensure that agents do not reveal information on the reward values. We designed algorithms for two multi-agent setting: the DP-Master-worker UCB algorithm and the DP-Decentralized UCB algorithm each capturing a different communication network connecting agents. We analyzed both the privacy and regret performance of these algorithms.
\newpage
\bibliography{isit2020}
\input{appendix_revise.tex}

%\begin{thebibliography}{00}
%\bibitem{b1} G. Eason, B. Noble, and I. N. Sneddon, ``On certain integrals of Lipschitz-Hankel type involving products of Bessel functions,'' Phil. Trans. Roy. Soc. London, vol. A247, pp. 529--551, April 1955.
%\bibitem{b2} J. Clerk Maxwell, A Treatise on Electricity and Magnetism, 3rd ed., vol. 2. Oxford: Clarendon, 1892, pp.68--73.
%\bibitem{b3} I. S. Jacobs and C. P. Bean, ``Fine particles, thin films and exchange anisotropy,'' in Magnetism, vol. III, G. T. Rado and H. Suhl, Eds. New York: Academic, 1963, pp. 271--350.
%\bibitem{b4} K. Elissa, ``Title of paper if known,'' unpublished.
%\bibitem{b5} R. Nicole, ``Title of paper with only first word capitalized,'' J. Name Stand. Abbrev., in press.
%\bibitem{b6} Y. Yorozu, M. Hirano, K. Oka, and Y. Tagawa, ``Electron spectroscopy studies on magneto-optical media and plastic substrate interface,'' IEEE Transl. J. Magn. Japan, vol. 2, pp. 740--741, August 1987 [Digests 9th Annual Conf. Magnetics Japan, p. 301, 1982].
%\bibitem{b7} M. Young, The Technical Writer's Handbook. Mill Valley, CA: University Science, 1989.
%\end{thebibliography}
\vspace{12pt}
%\color{red}

\end{document}

%% file: appendix_revise.tex
\section{Appendix}
\subsection{Useful facts}
\textbf{Fact 1} \textit{(Hybrid Mechanism, Corollary 4.8 in \cite{chan2011private}) 
	The hybrid mechanism is $\epsilon$- differential private and has an error bounded with probability at least $1-\delta$ by $\frac{1}{\epsilon}\cdot (\log~t)^{1.5}\cdot \log \frac{1}{\delta}$ at time $t$. }

\textbf{Fact 2} \textit{(Chernoff-Hoeffding bound)
Let $X_1,...,X_t$ be a sequence of real-valued random variables with common range $[0,1]$, and such that $\mathbb{E}[X_t|X_1,...,X_{t-1}]=\mu$. Let $S_t = \sum_{i=1}^{t}X_t$. Then for all $a\geq 0$,}
\begin{equation}
P(S_t\geq t\mu ~ +~a)~\leq~e^{-2a^2/t},P(S_t\leq t\mu ~ +~a)~\leq~e^{-2a^2/t}\notag
\end{equation}

\subsection{Lemmas}

\begin{lemma} [Estimated Performance of Algorithm 2]	\quad\quad \quad \quad
	i) The estimated number of plays satisfies: 
	%\begin{equation}\label{ep1}
	$$n_{j}^{\text{avg}}(t)-c_0 \leq\hat{n}_{i,j}(t)\leq n_{j}^{\text{avg}}(t)+c_0;$$
	%\end{equation}
	ii) $\hat{ {Y}}_{i,j} $ is an unbiased estimation with
	%\begin{equation}\label{ep2}
	$\mathbb{E}[\hat{ {Y}}_{i,j}(t)]=\mu_j$;
	%\end{equation} 
	\\iii) The variance of $\hat{ {Y}}_{i,j}(t)$ is bounded by :
	%\begin{equation}\label{ep3}
	$$\text{Var}[\hat{ {Y}}_{i,j}(t)]\leq\frac{\hat{n}_{i,j}(t)+c_i}{M\hat{n}_{i,j}(t)^2}\sigma^2;$$
	%\end{equation}
	iv) The error between $\hat{ {Y}}_{i,j}(t)$ and $\hat{ {X}}_{i,j}(t)$ can be bounded by: $$h_{\hat{n}_{i,j}(t)}=\frac{1}{\epsilon}\cdot \log ^{1.5}({\hat{n}_{i,j}})\cdot \log \frac{1}{\delta}\cdot \frac{1}{\hat{n}_{i,j}}$$
	\\	
	where $n_{i,j}^{\text{avg}}(t)=\frac{1}{M}\sum_{\tau=1}^{t}\mathbf{1}_M^T\mathbf{\xi}_j(\tau)$ is the total number of times arm $j$ has been pulled by agent $i$ up to time $t$.
\end{lemma}

\begin{proof}
Let $\lambda_j$ be the $j$ th largest eigenvalue of $P$, $\mathbf{u}_j$ be the eigenvector corresponding to $\lambda_j$.

Then we have $\lambda_1=1$ and $\mathbf{u}_1=\mathbf{1}_M/\sqrt{M}$. We define
\begin{eqnarray}
\nu_{pk}^{+sum} = \sum_{d=1}^{M}u_p^du_k^d\mathbbm {1}((\mathbf{u}_p\mathbf{u}_k^\top )_{ii}\geq 0)\\
\nu_{pk}^{-sum} = \sum_{d=1}^{M}u_p^du_k^d\mathbbm {1}((\mathbf{u}_p\mathbf{u}_k^\top )_{ii}\leq 0)
\end{eqnarray}

To present the topology of communication graph, we denote two parameters $c_0$ and $c_i$ as:
\begin{eqnarray}
c_0 &=& \sqrt{M}\sum_{p=2}^{M}\frac{|\lambda_p|}{1-|\lambda_p|}\\
c_i &=& M\sum_{p=1}^{M}\sum_{k=2}^{M}\frac{|\lambda_p||\lambda_k|}{1-|\lambda_p||\lambda_k|}a_{pk}(i)
\end{eqnarray}
Where 
\begin{eqnarray}
a_{pk}(i)=\left\{\begin{matrix}
\nu_{pk}^{+sum}(\mathbf{u}_p\mathbf{u}_k^\top )_{kk}~,\text{if}~\lambda_p\lambda_k\geq 0~\&~ (\mathbf{u}_p\mathbf{u}_k^\top )_{ii}\geq 0\\ 
\nu_{pk}^{-sum}(\mathbf{u}_p\mathbf{u}_k^\top )_{ii}~,\text{if}~\lambda_p\lambda_k\geq 0~\&~ (\mathbf{u}_p\mathbf{u}_k^\top )_{ii}\leq 0\\ 
\max\left\lbrace |\nu_{pk}^{-sum}|,\nu_{pk}^{+sum}\right\rbrace, \text{if} ~\lambda_p\lambda_k<0
\end{matrix}\right.\notag
\end{eqnarray}
We start with the first statement. From \eqref{eq9} it follows that,
\begin{eqnarray}
\hat{\mathbf{n}}_j(t)&=&P^t\hat{\mathbf{n}}_j(0)+\sum_{\tau=0}^{t-1}P^{t-\tau}\mathbb{\xi}_j(\tau)\notag\\
&=&\sum_{\tau=0}^{t-1}[\frac{1}{M}\mathbf{1}_M\mathbf{1}_M^\top\mathbb{\xi}_j(\tau)+\sum_{p=2}^{M}\lambda_p^{t-\tau}\mathbf{u}_p\mathbf{u}_p^\top\mathbb{\xi}_j(\tau)]\notag\\
\label{eq20}&=&n_j^{\text{avg}}(t)\mathbf{1}_M
+\sum_{\tau=0}^{t-1}\sum_{p=2}^{M}
\lambda_p^{t-\tau}\mathbf{u}_p\mathbf{u}_p^\top\mathbb{\xi}_j(\tau)
\end{eqnarray}
We now bound the second term on the right hand of \eqref{eq20}:
\begin{eqnarray}
\sum_{\tau=0}^{t-1}\sum_{p=2}^{M}
\lambda_p^{t-\tau}\mathbf{u}_p\mathbf{u}_p^\top\mathbb{\xi}_j(\tau)
\leq \sum_{\tau=0}^{t-1}\sum_{p=2}^{M}
|\lambda_p^{t-\tau}|\left \|\mathbf{u}_p  \right \|_2^2\left \|\mathbb{\xi}_j(\tau) \right \|_2\notag\\
\leq\sqrt{M}\sum_{\tau=0}^{t-1}\sum_{p=2}^{M}|\lambda_p^{t-\tau}|\leq c_0
\end{eqnarray}
This complete the statement i).

Similarly, for statement ii) we have:
\begin{equation}
\hat{\mathbf{y}}_j(t)=P^t\hat{\mathbf{s}}_j(0)+\sum_{\tau=0}^{t-1}P^{t-\tau}\mathbf{r}_j(\tau)=\sum_{\tau=0}^{t-1}=P^{t-\tau}\mathbf{r}_j(\tau)
\end{equation}
Calculate expected on both side of above equality we have $\mathbb{E}[\hat{\mathbf{y}}_j(t)]=\mu_j\sum_{\tau = 0}^{t-1}\mathbb{\xi}_j(\tau)=\mu_i\hat{\mathbf{n}}_j(t)$, this proof statement ii).

For the third statement, we have
\begin{eqnarray}\label{eq23}
\text{Cov}[\hat{\mathbf{y}}_j(t)]&=&(P^{t-\tau})\Sigma(\tau)\Sigma(\tau)^\top (P^{t-\tau})^\top\notag\\
&=&\sum_{\tau=0}^{t-1}\sum_{p=1}^{M}\sum_{k=1}^{M}\lambda_p^{t-\tau}\lambda_j^{t-\tau}\mathbf{u}_p\mathbf{u}_p^\top\Sigma(\tau)\Sigma(\tau)^\top\mathbf{u}_k\mathbf{u}_k^\top\notag\\
&=& \underset{\textcircled{1}}{\underbrace{ \sigma^2\sum_{\tau=0}^{t-1}\sum_{p=1}^{M}\sum_{k=1}^{M}(\lambda_p\lambda_k)^{t-\tau}\varsigma_{pk}(\tau)(\mathbf{u}_p\mathbf{u}_k^\top)}}\notag\\
&+&\underset{\textcircled{2}}{\underbrace{\frac{1}{M}\sum_{\tau=0}^{t-1}\sum_{p=1}^{M}\lambda_p^{t-\tau}\mathbf{u}_p\mathbf{u}_k^\top\Sigma(\tau)\Sigma(\tau)^\top\mathbf{1}_M\mathbf{1}_M^\top}}
\end{eqnarray}
Where $\Sigma(\tau)=\sigma\text{diag}(\mathbb{\xi}_j(\tau))$, $\varsigma_{pk}(\tau)=\mathbf{u}_p^\top\text{diag}(\mathbb{\xi}_j(\tau))\mathbf{u}_k$.
We examine the $ii-$th entry of \eqref{eq23} and define \textcircled{1} and \textcircled{2} as the $ii$-th entry of the first term and second term in \eqref{eq23} respectively.
\begin{eqnarray}\label{eq24}
\textcircled{1} &\leq& \sigma^2\sum_{\tau=0}^{t-1}\sum_{p=1}^{M}\sum_{k=1}^{M}|\lambda_p\lambda_k|^{t-\tau}|\varsigma_{pk}(\tau)(\mathbf{u}_p\mathbf{u}_k^\top)_{ii}|\notag\\
&\leq& \sigma^2\sum_{\tau=0}^{t-1}\sum_{p=1}^{M}\sum_{k=1}^{M}|\lambda_p\lambda_k|^{t-\tau}|a_{pk}(i)\notag\\
&\leq&\sum_{p=1}^{M}\sum_{k=2}^{M}\frac{|\lambda_p||\lambda_k|}{1-|\lambda_p||\lambda_k|}a_{pk}(i)=\sigma^2\frac{c_i}{M}
\end{eqnarray}
\begin{eqnarray}\label{eq25}
\textcircled{2} =\frac{1}{M}[(\sum_{\tau=0}^{t-1}\sum_{p=1}^{M}\lambda_p^{t-\tau}\mathbf{u}_p\mathbf{u}_k^\top\mathbf{\xi}_j(\tau))\mathbf{1}_M^\top]_{ii}=\sigma^2\frac{\hat{n}_{i,j}(t)}{M}
\end{eqnarray}
Combining \eqref{eq24} and \eqref{eq25}, we establish statement iii).

Directly following Lemma 1, we can proof statement iv). This complete the proof of Lemma 2.
\end{proof}

\subsection{Proof of Theorem 2}
\begin{proof}
Using Lemma 1, we rewrite the bound into following equations:
\begin{eqnarray}
\label{eq1}P(X\geq Y+h_n)\leq \delta\\
\label{eq2}P(X\leq Y-h_n)\leq \delta
\end{eqnarray}

%\textbf{Lemma 3} The expected regret of UCB algorithm over time horizon $T$ is given by:
%\begin{equation}
%R_{UCB}(T)=\sum_{j=2}^{K}\frac{8\log T}{\Delta_j}+(1+\frac{\pi^2}{3})\sum_{j=2}^{K}\Delta_j
%\end{equation}
%If the algorithm recompute the index once every $L$ time slots instead of each time slots, the regret is given by:
%\begin{equation}
%R_{UCB}(T)=\sum_{j=2}^{K}\frac{8L\log T}{\Delta_j}+L(1+\frac{\pi^2}{3})\sum_{j=2}^{K}\Delta_j
%\end{equation}
%\textit{Proof}: The proof can be proved by observing that if a suboptimal arm is selected, it will be played for the next $L$ time slots. For our Algorithm 1, $L$ can be seen like the frequency the agents communicate with the central node, however, we use the dynamic communication frequency instead of a fixed $L$. Specifically, the recomputation is decided by whether the selected best arm is same to the last choice. As an arm is fully played (after a fully exploration phase), we will have an accurate estimate of the unknown mean. Therefore, this arm will enter the exploration phase (keep playing the same best arm until $T$). In this stage, it is unnecessary to communicate with the central node anymore which will lead to computing or communication overhead.

Recall $\epsilon$ is the is the differential privacy parameter. The regret incurred during time horizon $T$ can be analyzed as the sum of the regret caused by playing suboptimal arms and recomputing the arm index by federated learning. We denote $n_{i,j}(T)$ as the times that a suboptimal arm $j$ is played by agent $i$, $c_{t,n} = \sqrt{\frac{\log T}{2}}$ as the UCB confidence index. Our proof steps mostly follows the demonstration of UCB algorithm\cite{auer2002finite}. Considering a suboptimal arm $j\geq1$ of player $i$, let $\tau_{i,j}(m)$ be the time that play make the $m$-th switch to arm $j$ and $\tau'_{i,j}(m)$ be the time that the player leave arm $j$ and turn another one. Then, we have, 
\begin{equation}\label{subnum}
\begin{array}{ll}
n_{i,j}(T)\notag\\
\leq 1+\sum_{m=1}^{T}\left| \tau'_{i,j}(m)-\tau_{i,j}(m)\right|
I\left\lbrace \text{Play arm}~j~ \text{at time~}\tau_{i,j}(m)\right\rbrace \notag\\
\leq 1+\sum_{m=1}^{T}\left| \tau'_{i,j}(m)-\tau_{i,j}(m)\right|\notag\\
I\left\lbrace \sum_{i=1}^{M}I_{i,j}(\tau_{i,j}(m)-1)\geq \sum_{i=1}^{M}I_{i,1}(\tau_{i,j}(m)-1)\right\rbrace\notag\\
\leq l+\sum_{m=1}^{T}\sum_{l=0}^{\infty}2^pI\left\lbrace{ \sum_{i=1}^{M}I_{i,j}(\tau_{i,j}(m)+2^p-2)\geq}\right.\notag\\ \phantom{}\left.{ \sum_{i=1}^{M}I_{i,1}(\tau_{i,j}(m)+2^p-2),n_{i,j}(\tau_{i,j}(m)-1)\geq l}\right\rbrace\notag\\
\leq l+\sum_{m=1}^{T}\sum_{p=0}^{\infty}2^pI\left\lbrace{ \sum_{i=1}^{M}I_{i,j}(m+2^p-2)\geq}\right.\notag\\ \phantom{}\left.{ \sum_{i=1}^{M}I_{i,1}(m+2^p-2),n_{i,j}(m-1)\geq l}\right\rbrace\notag\\
\leq l+\sum_{m=1}^{\infty}\sum_{m+2^p\leq T}\leq 2^p \sum_{n_{i,1}=1}^{m+2^p}\sum_{n_{i,j}=l}^{m+2^p}\notag\\
I\left\lbrace {\sum_{i=1}^{M}( {X}_{i,j}(m+2^p)+ c_{m+2^p,n_{i,j}}+h_{n_{i,j}})\geq}\right.\notag\\ \phantom{}\left.{ \sum_{i=1}^{M}( {X}_{i,1}(m+2^p)+c_{m+2^p,n_{i,1}}+h_{n_{i,1}})}\right\rbrace\notag
\end{array}
\end{equation}

In Algorithm 1, if an arm is for the pth time consecutively (without switching to any other arms in between), it will be played for the next $2^p$ slots. The second inequality uses this fact. In the second last inequality , we replace $ \tau_{i,j}(m)$ by $m$ which is clearly an upper bound. 

It should be noted that each time step $t$ when index updates, we have $I_{1,j}(t)=I_{2,j}(t)=...=I_{i,j}(t)=...=I_{M,j}(t)=\frac{1}{M}\sum_{i=1}^{M}I_{i,k}(t)$. That means all agents select the same arm according to the central update results, so in equation \eqref{subnum}, we can observe that the event $\sum_{i=1}^{M}({X}_{i,j}(m+2^p)+c_{m+2^p,n_{i,j}}+h_{n_{i,j}})\geq \sum_{i=1}^{M}( {X}_{i,1}(m+2^p)+c_{m+2^p,n_{i,1}}+h_{n_{i,1}})$ implies that for each player $i$ at least one of the following events holds:
\begin{eqnarray}
\label{ev1}{X}_{i,1}(m+2^p)\leq\mu_{i,1}-c_{m+2^p,n_{i,1}}-h_{n_{i,1}}\\
\label{ev2}{X}_{i,j}(m+2^p)\geq\mu_{i,j}+c_{m+2^p,n_{i,j}}+h_{n_{i,j}}\\
\label{ev3}\mu_{i,1}\le \mu_{i,j}+2c_{m+2^p,n_{i,j}}+2h_{n_{i,j}}
\end{eqnarray}

Now, using the Chernoff-Hoeffding bound, we can get:
\begin{eqnarray}
\Pr(\eqref{ev1}) &=& \Pr({X}_{i,1}(m+2^p)\leq\mu_{i,1}-c_{m+2^p,n_{i,1}}-h_{n_{i,1}})\notag\\
&=&\Pr({X}_{i,1}(m+2^p)\leq {Y}_{i,1}(m+2^p)-h_{n_{i,1}}\notag\\
 &\vee&  {Y}_{i,1} (m+2^p)\leq\mu_{i,1}-c_{m+2^p,n_{i,1}})\notag\\
&=&\Pr({X}_{i,1}(m+2^p)\leq {Y}_{i,1}(m+2^p)-h_{n_{i,1}})\notag\\
&+& \Pr(  {Y}_{i,1} (m+2^p)\leq\mu_{i,1}-c_{m+2^p,n_{i,1}})\notag\\
&=&\delta + (m+2^p)^{-4}
\end{eqnarray} 

Similarly, we can use \eqref{eq1} and the Chernoff-Hoeffding bound to prove a bound on \eqref{ev2}:

\begin{eqnarray}
\Pr (\eqref{ev2}) &=& \Pr({X}_{i,j}(m+2^p)\geq\mu_{i,1}+c_{m+2^p,n_{i,j}}+h_{n_{i,j}})\notag\\
&=&\Pr({X}_{i,j}(m+2^p)\leq {Y}(m+2^p)_{i,j}+h_{n_{i,1}} \notag\\
&\vee&  {Y}_{i,j} (m+2^p)\leq\mu_{i,j}+c_{m+2^p,n_{i,j}})\notag\\
&=&Pr({X}_{i,j}(m+2^p)\leq {Y}_{i,j}(m+2^p)+h_{n_{i,j}}) \notag\\
&+& \Pr(  {Y}_{i,j} (m+2^p)\leq\mu_{i,j}-c_{m+2^p,n_{i,j}})\notag\\
&=&\delta + (m+2^p)^{-4}
\end{eqnarray} 

To prove a bound on \eqref{ev3}, we try to find a minimum number $n_{i,j}$ for which \eqref{ev3} is always false. This event is flase indicates that $\Delta_{i,j}\geq 2c_{m+2^p,n_{i,j}}+2h_{n_{i,j}}$, which implies the following two equations hold for any $
0\leq\beta_0\leq 1 $.
\begin{eqnarray}
\label{ev4}\beta_0 \Delta_{i,j}\ge  2c_{m+2^p,n_{i,j}}\\
\label{ev5}(1-\beta_0)\Delta_{i,j}\ge 2h_{n_{i,j}}
\end{eqnarray}

For $n_{i,j}\geq \left \lceil \frac{8\log T}{\Delta_{i,j}^2 \beta_0^2} \right \rceil$, \eqref{ev4} is false. Equation $\eqref{ev5}$ implies that 
\begin{equation}\label{pev5}
n_{i,j}\geq \beta_1\cdot \log (n_{i,j})^{1.5}
\end{equation}
where
\begin{eqnarray}
\beta_1 = \frac{2}{(1-\beta_0) \Delta_{i,j}}\cdot \frac{1}{\epsilon}\log{\frac{1}{\delta}}
\end{eqnarray}
Let $x=-\log (n_{i,j}/1.5)$, above inequality can be rewrite in the standard transcendental algebraic inequality form:
\begin{equation}
e^{-x}\geq -1.5x\beta_1^{\frac{1}{1.5}}
\end{equation}
The solution can be given by the Lambert W function, so 
\begin{eqnarray}
n_{i,j} &\geq& \exp(-1.5(W(-1,\frac{-1}{1.5\beta_1^{\frac{1}{1.5}}})))\notag\\
&\approx& \beta_1^{2.25}=(\frac{2}{(1-\beta_0) \Delta_{i,j}}\cdot \frac{1}{\epsilon}\log {\frac{1}{\delta}})^{2.25}
\end{eqnarray}
Here we choose $\delta = (m+2^p)^{-4}\leq T^{-4}$,that yields 
\begin{equation}
n_{i,j}\geq \max[(\frac{8}{(1-\beta_0) \Delta_{i,j}}\cdot \frac{1}{\epsilon}\log{T})^{2.25}, \frac{8\log T}{\Delta_{i,j}^2 \beta_0^2} ]
\end{equation}

In summary, the total number of suboptimal plays is 
\begin{equation}
\begin{array}{ll}
\mathbb{E}[n(T)] = \sum_{i=1}^{M}\sum_{j>1}^{K}\mathbb{E}[n_{i,j}(T)]\notag\\
=\sum_{i=1}^{M}\sum_{j>1}^{K}[\max[(\frac{8}{(1-\beta_0) \Delta_{i,j}}\cdot \frac{1}{\epsilon}\log T)^{2.25},
 \phantom{}\left \lceil  \frac{8\log T}{\Delta_{i,j}^2 \beta_0^2} \right \rceil]]\notag\\
+\sum_{m=1}^{\infty}\sum_{p=0}^{\infty}2^p\sum_{n_{i,1}=1}^{m+2^p}\sum_{n_{i,j}=1}^{m+2^p}4(m+2^p)^{-4}\notag\\
\leq \sum_{i=1}^{M}\sum_{j>1}^{K}(4+\max[(\frac{8}{(1-\beta_0) \Delta_{i,j}}\cdot \frac{1}{\epsilon}\log T)^{2.25},
 \phantom{}\left \lceil \frac{8\log T}{\Delta_{i,j}^2 \beta_0^2} \right \rceil])\notag\\
\leq MK(4+\max[(\frac{8}{(1-\beta_0) \Delta_{i,j}}\cdot \frac{1}{\epsilon}\log T)^{2.25},
 \phantom{}\left \lceil \frac{8\log T}{\Delta_{i,j}^2 \beta_0^2}\right \rceil])
 \end{array}
\end{equation}

%Consider there exists a computation cost $C$ every time the agent communicate with the central node. Thus, we give a bound of the number of arm index recomputation $m(T)$, then the regret can be calculated as $\Delta_{max}\mathbb{E}(m(T))$. We define $m(T) = m_1(T)+m_2(T)$, where $m_1(T)$ is the number of index updates leading to a optimal selection, $m_2(T)$ is the number of index updates leading to a suboptimal selection. For $m_2(T)$ the number of sunoptimal index updated is less than the number of times the suboptimal arm is played:
%\begin{equation}
%\mathbb{E}[m_2(T)]\leq \sum_{i=1}^{M}\sum_{j>1}^{K} \mathbb{E}[n_{i,j}(T)]
%\end{equation}
%To bound $\mathbb{E}[m_1(T)]$, we make $\tau_l$ be the time the player make the $l$-th transition to an optimal arm from others, and at $\tau_l'$ the player make the $l$-th transition to other suboptimal arms.Then we have, 
%\begin{equation}
%m_1(T)\leq \sum_{l=1}^{n(T)}\log|\tau_l-\tau_l'|
%\end{equation}
%Here, $n(T)$ is the total number of times the suboptimal arms is played and $\log|\tau_l-\tau_l'|\leq \log T$
%\begin{equation}
%\mathbb{E}[m_1(T)]\leq \sum_{i=1}^{M}\sum_{j>1}^{K} \mathbb{E}[n_{i,j}(T)]\cdot \log T
%\end{equation}
%Now we can have $\mathbb{E}[m(T)]$,
%\begin{equation}
%\mathbb{E}[m(T)]\leq \sum_{i=1}^{M}\sum_{j>1}^{K} \mathbb{E}[n_{i,j}(T)]\cdot (1+\log T) 
%\end{equation}

The expected regret is 
%\begin{eqnarray}
%R^C(T) &=&\mathbb{E}(n(T)) + C\mathbb{E}(m(T))\notag\\
%&\leq&\Delta_{max}\mathbb{E}(n(T)) + C\mathbb{E}(m(T))\notag\\
%&\leq& (\Delta_{max}+C(1+\log T))\mathbb{E}(n(T))
%\end{eqnarray}
%\begin{eqnarray}
%R(T) &=&\Delta_{max}\mathbb{E}(n(T))
%\end{eqnarray}

%The expected regret is 
\begin{eqnarray}
 &R^C(T) \leq\Delta_{max}\mathbb{E}(n(T))\notag\\
&\leq \Delta_{max} MK(4+\max[(\frac{8}{(1-\beta_0) \Delta_{i,j}}\cdot %\frac{1}{\epsilon}\log T)^{2.25},
\phantom{}\left \lceil \frac{8\log T}{\Delta_{i,j}^2 \beta_0^2}\right \rceil])
\end{eqnarray}
This completes the proof.
\end{proof}
\subsection{Proof of Theorem 4}

\begin{proof}
The proof is similar to Theorem 2. The regret incurred can be analyzed by playing the suboptimal arms ($j\geq 1$) during time horizon $T$:
\begin{equation}
\begin{array}{ll}
n(T)=1 + \sum_{i=1}^{M}\sum_{t=K+1}^{T}n_{i,j}(t)\notag\\
%&=& \sum_{i=1}^{M}\sum_{t=1}^{T}I(suboptimal~arm~j~is selected)\notag\\
= 1 + \sum_{i=1}^{M}\sum_{t=1}^{T}I\{a_i(t)=j\}\notag\\
\leq l + \sum_{i=1}^{M}\sum_{t=1}^{T}I\{I_{i,j}(t)\geq I_{i,1}(t),n_j^{\text{avg}}\geq l\}\notag\\
\leq l+\sum_{t=1}^{T}I\left\lbrace {\sum_{i=1}^{M}\hat{ {X}}_{i,j}(t)+c_{t,n_{i,j}}+h_{n_{i,j}}} \right.\notag\\ \phantom{}\left.{\geq \sum_{i=1}^{M}\hat{ {X}}_{i,1}(t)+c_{t,n_{i,1}}+h_{n_{i,1}},n_j^{\text{avg}}\geq l }\right\rbrace
\end{array}
\end{equation}
At time slot $t$, the individual agent $i$ will choose a suboptimal arm only if the event $\left\lbrace {\sum_{i=1}^{M}\hat{ {X}}_{i,j}(t)+c_{t,\hat{n}_{i,j}+h_{\hat{n}_{i,j}}} \geq \sum_{i=1}^{M}\hat{ {X}}_{i,1}(t)+c_{t,\hat{n}_{i,1}+h_{\hat{n}_{i,1}}} }\right\rbrace $ holds.
It indicated that at least one of the following three conditions must holds:
\begin{eqnarray}
\label{ev6}\hat{ {X}}_{i,1}\leq \mu_1-c_{t,\hat{n}_{i,1}}-h_{\hat{n}_{i,1}}\\
\label{ev7}\hat{ {X}}_{i,j}\leq \mu_j+c_{t,\hat{n}_{i,j}}+h_{\hat{n}_{i,j}}\\
\label{ev8}\mu_1< \mu_j+2c_{t,\hat{n}_i,j}+2h_{\hat{n}_{i,j}}
\end{eqnarray}
According to Lemma 2, $c_{t,\hat{n}_{i,j}}=\sigma\sqrt{\frac{\hat{n}_{i,j}(t)+c_i}{M\hat{n}_{i,j}(t)}\cdot \frac{2\rho\log T}{\hat{n}_{i,j}(t)}}$.
Using the Fact 2, the union bound and the Chernoff-Hoffding bound, we can prove the probability of \eqref{ev6} is:
\begin{eqnarray}
\Pr(\eqref{ev6})&=&Pr(\hat{ {X}}_{i,1}\leq \mu_1-c_{t,\hat{n}_{i,1}}-h_{\hat{n}_{i,1}})\notag\\
&=&\Pr (\hat{ {X}}_{i,1}\leq\hat{ {Y}}_{i,1}-h_{n_{i,1}})+\notag\\
&\phantom{}&\Pr\left( z\geq \frac{\mathbb{E}[\hat{ {Y}}_{i,1}]+c_{t,\hat{n}_{i,1}}-\mu_1}{\sqrt{\text{Var}(\hat{ {Y}}_{i,1})}}\right) \notag\\
&\leq&\delta+\Pr( z\geq \frac{c_{t,\hat{n}_{i,1}}}{\sqrt{\text{Var}(\hat{ {Y}}_{i,1})}})\notag\\
&\leq& \delta+\frac{1}{2}\exp(-\frac{c_{t,\hat{n}_{i,1}^2}}{2\text{Var}(\hat{ {Y}}_{i,1})})\leq\delta+\frac{1}{2t^\rho}
\end{eqnarray}
where $z$ is the standard Gaussian random variable. The last inequality follows from the tail bounds for the error function and the statement ii) of Lemma 2. 
Similarly, we can prove the bound of event \eqref{ev7}:
\begin{equation}
\Pr(\eqref{ev7}) \leq\delta+\frac{1}{2t^\rho}
\end{equation}

We can choose $\delta =\frac{1}{2t^\rho} $, which leading to :
\begin{eqnarray}
\Pr(\eqref{ev6})\leq t^{-\rho}\\
\Pr(\eqref{ev7})\leq t^{-\rho}
\end{eqnarray}

To prove a bound on \eqref{ev8}, we try to find a minimum number $\hat{n}_{i,j}$ for which \eqref{ev8} is always false. This event is false indicates that $\Delta_{i,j}\geq 2c_{t,n_{i,j}}+2h_{n_{i,j}}$, which implies the following two equations hold for any $
0\leq\beta_0\leq 1 $.
\begin{eqnarray}
\label{ev11}\beta_0 \Delta_{i,j}\ge  2c_{t,\hat{n}_{i,j}}\\
\label{ev12}(1-\beta_0)\Delta_{i,j}\ge 2h_{\hat{n}_{i,j}}
\end{eqnarray}

For  $l = \left \lceil \frac{c_0}{\beta_0^2}+\frac{8\sigma^2\rho(1+c_i)\log T}{M\beta_0^2\Delta_{i,j}^2} \right \rceil$, \eqref{ev11} is false. 

The appropriate choice of $n_{i,j}$ for holding $\eqref{ev12}$ can be same of Theorem 2 since \eqref{ev12} is only related to the DP error.
\begin{eqnarray}
\hat{n}_{i,j} &\geq& \beta_1^{2.25}=(\frac{2}{(1-\beta_0) \Delta_{i,j}}\cdot \frac{1}{\epsilon}\log {\frac{1}{\delta}})^{2.25}
\end{eqnarray}

Since we choose $\delta = \frac{1}{2}t^{-\rho}$,that yields 
\begin{equation}
\begin{array}{ll}
\hat{n}_{i,j}\geq \max[(\frac{2+2\rho \log T }{\epsilon(1-\beta_0)\Delta_{i,j}})^{2.25}, 
\phantom{}\left \lceil \frac{c_0}{\beta_0^2}+\frac{8\sigma^2\rho(1+c_i)\log T}{M\beta_0^2\Delta_{i,j}^2} \right \rceil]
\end{array}
\end{equation}

In summary, the total number of suboptimal plays is 
\begin{equation}
\begin{array}{ll}
\mathbb{E}[n(T)]=\notag\\%&=& \sum_{i=1}^{M}\sum_{j>1}^{K}\mathbb{E}[n_{i,j}(T)]\notag\\
\sum_{i=1}^{M}\sum_{j>1}^{K}[\max[(\frac{2+2\rho \log T }{\epsilon(1-\beta_0)\Delta_{i,j}})^{2.25},
\phantom{}\left \lceil \frac{c_0}{\beta_0^2}+\frac{8\sigma^2\rho(1+c_i)\log T}{M\beta_0\Delta_{i,j}^2} \right \rceil]]\notag\\
\phantom{}+\sum_{t=1}^{T}\sum_{n_{i,1}=1}^{t}\sum_{n_{i,j}=1}^{t}2t^{-\rho}\notag\\
\leq \sum_{i=1}^{M}\sum_{j>1}^{K}(\frac{2\rho}{\rho -1}+\max[(\frac{2+2\rho \log T }{\epsilon(1-\beta_0)\Delta_{i,j}})^{2.25}, 
\phantom{}\left \lceil\frac{c_0}{\beta_0^2}+\frac{8\sigma^2\rho(1+c_i)\log T}{M\beta_0\Delta_{i,j}^2} \right \rceil])\notag\\
\leq \frac{2MK\rho}{\rho -1}+\sum_{i=1}^{M}\sum_{j>1}^{K}\max[(\frac{2+2\rho \log T }{\epsilon(1-\beta_0)\Delta_{i,j}})^{2.25},
\phantom{}\left \lceil \frac{c_0}{\beta_0^2}+\frac{8\sigma^2\rho(1+c_i)\log T}{M\beta_0^2\Delta_{i,j}^2} \right \rceil]
\end{array}
\end{equation}
Using $R^D(T)=\Delta_{max}\cdot\mathbb{E}[n(T)]$, we can complete the proof.
\end{proof}